\documentclass[12pt]{colt2019} 

\usepackage{times}

\newtheorem{theorem-new}{Theorem}
\usepackage{graphicx,psfrag,dsfont}
\usepackage{amsfonts,amsmath,amssymb,color}
\usepackage{bbm,setspace}

\usepackage{xspace}
\usepackage{bbm}

\def\var{\mathop{\rm var}\nolimits}%
\def\Var{\mathop{\rm var}\nolimits}%
\newcommand{\Cc}{\mathcal{C}}

\newcommand{\Nc}{\mathcal{N}}

\newcommand{\Sc}{\mathcal{S}}
\newcommand{\Tc}{\mathcal{T}}

\newcommand{\Xc}{\mathcal{X}}
\newcommand{\Yc}{\mathcal{Y}}

\newcommand{\wv}{{\bf w}}
\newcommand{\xv}{{\bf x}}
\newcommand{\rv}{{\bf r}}
\newcommand{\yv}{{\bf y}}
\newcommand{\zv}{{\bf z}}

\DeclareMathOperator\E{E}
\let\P\relax
\DeclareMathOperator\P{P}

\newcommand{\N}{\mathrm{N}}

\def\textiid{i.i.d.\@\xspace}
\newcommand\iid{\ifmmode\text{ i.i.d. } \else \textiid \fi}

\newcommand{\Real}{\mathbb{R}}

\newcommand{\norm}[1]{\|#1\|}
\newcommand{\normsq}[1]{\norm{#1}^2}

\newcommand{\mvec}[1]{{\boldsymbol #1}}

\newcommand{\ex}{{\rm e}}

\newcommand{\av}{\mvec{a}}

\newtheorem{assumption}{Assumption}

\DeclareMathOperator*{\argmax}{arg\,max}

\title[Efficient Deep Learning of GMMs]{Efficient Deep Learning of GMMs}

\coltauthor{%
\begin{center}
\Name{Shirin Jalali,\hspace{0.5em}Carl Nuzman,\hspace{0.5em}Iraj Saniee}\\
\addr{\{shirin.jalali,carl.nuzman,iraj.saniee\}@nokia-bell-labs.com}\\
\addr Bell Labs, Nokia, 600-700 Mountain Avenue, Murray Hill, NJ 07974
\end{center}
}

\begin{document}
\maketitle

\begin{abstract}
%

We show that a collection of Gaussian mixture models (GMMs) in $\Real^{n}$ can be optimally classified using $O(n)$ neurons in a neural network with two hidden layers (deep neural network), whereas in contrast, a neural network  with a single hidden layer (shallow neural network)
would require at least $O(\exp(n))$ neurons or possibly exponentially large coefficients. Given the universality of  the Gaussian distribution in the feature spaces of data, e.g., in speech, image and text, 
our result  sheds light on the observed efficiency of deep neural  networks in practical classification problems.  
\end{abstract}


\vspace{10mm}
\section{Introduction} 
There is a rapidly growing literature which demonstrates the effectiveness of deep 
neural networks in classification problems that arise in practice; e.g., for classification 
of audio, image and text data sets.  The universal approximation 
theorem, UAT (\cite{cybenko1989approximations,funahashi1989approximate}),
states that any regular function, which 
for example, separates in the (high dimensional) feature space 
a collection of points corresponding to images of dogs from those of cats, 
can be approximated by a neural network.  But UAT is proven for shallow, i.e., single hidden-layer,
neural networks and in fact the number of neurons needed may be exponentially or super 
exponentially large in the size of the feature space of the data. 
Yet, practical deep neural networks are able to solve such 
classification problems effectively and efficiently, i.e., using what amounts to a small number 
of neurons in terms of the size of the feature space of the data.  There is no theory 
yet as to why deep neural networks (DNNs from here on) are as effective and efficient 
in practice as they evidently are.  
There are essentially two possibilities for this observed outcome: 1) DNNs are 
\textit{always} significantly more efficient in terms of the number of neurons used for approximation 
of \textit{any} relevant functions than shallow networks, or 2) discriminant functions that arise 
\textit{in practice}, e.g., those that separate in the feature space of images points representing 
dogs from points representing cats, are particularly suited for DNNs.  If the latter proposition is true,
then the observed efficiency of DNNs is essentially due to the special form of the 
discriminant functions encountered in practice and not to the universal efficiency of DNNs,
which the former proposition would imply.

The first alternative proposed above is a general question about function 
approximation given 
neural networks as the collection of basis functions.  We are not aware of 
general 
results that show DNNs (those with two or more hidden layers) require
fundamentally fewer neurons for approximation of general functions than shallow neural networks 
(or SNNs from here on, i.e., those with a single hidden layer).  
In this paper, we focus on the second alternative and provide an answer in the affirmative; 
that indeed many discriminant functions that arise in practice are such that DNNs
require significantly, e.g., logarithmically, fewer neurons for their approximation than SNNs. To
formalize what may constitute discriminant functions that arise in practice, we focus on a versatile
class of distributions often used to model real-life distributions, namely Gaussian mixture model
(GMM for short), see Figure \ref{fig:3GMMs}. GMMs have been shown to be good models for audio, 
speech, image and text
processing in the past decades, e.g., see
\cite{gauvain1994maximum,reynolds2000speaker,portilla2003image,zivkovic2004improved,indurkhya2010handbook}.

\subsection{Background} 
The universal approximation theorem, see, for example, \cite{cybenko1989approximations,funahashi1989approximate,hornik1989multilayer,barron1994approximation}, teaches us that shallow neural networks 
(SNNs) can approximate regular functions to 
any required accuracy, albeit potentially with an exponentially large number of neurons.  
Can this number be reduced significantly, e.g., logarithmically, by deep neural networks?  
As indicated above, there is no such result as of yet and there is scant literature that 
even discusses this question.  Some evidence exists that DNNs may in fact not be efficient
in general, see \cite{abbe2018provable}.  On the other hand, some specialized functions have been 
constructed for which DNNs achieve significant and even logarithmic reduction in the
number of neurons compared to SNNs, 
e.g., see \cite{eldan2016power,rolnick2018the} for a certain radial 
function and polynomials, respectively. 
However, the functions considered in these references are 
typically very special and have little demonstrated basis in practice.  Perhaps the most 
illustrative cases are the high degree polynomials discussed in \cite{rolnick2018the} 
but the impressive logarithmic reduction in the number of neurons due to depth of the DNNs 
demonstrated in this work occurs only for very high degrees of 
polynomials in the feature coordinate size.

In this work we are motivated by model universality considerations.  What models of data are 
typical and what resulting discriminant functions do we typically need to approximate in 
practice? 
With a plausible model, we can determine if the resulting 
discriminant function(s) can be
approximated efficiently by deep networks.
To this end, we focus on 
data with Gaussian feature distributions, which provide a plausible and practical model for many 
types of data, especially when the feature space is sufficiently concentrated, e.g. after a
number of projections to lower-dimensional spaces, e.g., see \cite{bingham2001random}.
Our overall framework is based on the following set of definitions and demonstrations that 
we describe in detail:
\begin{itemize}
\item Definition of and notation for an $L$-deep neural network 
essentially consisting of a set of $L \geq 1$
affine transformations each with tunable coefficients and an additive translation (or bias), alternated 
with non-linear functions acting point-wise on each coordinate
\item A collection of (high-dimensional) GMMs, each consisting of a set of Gaussian distributions 
in dimension $n$ with arbitrary means and covariance matrices 
\item Definition of and notation for the (high-dimensional) classifier function for the above GMMs,
which is readily seen to be the maximum of multiple discriminant functions each
consisting of sums of exponentials of 
quadratic functions in dimension $n$
\item Definitions needed to link level of approximation of a set of discriminant functions with the performance of the corresponding classifier
\item Demonstration that DNNs can approximate general $n$-dimensional GMM discriminant functions 
using $O(n)$ neurons
\item Demonstration that SNNs need either an exponential (in $n$) number of 
neurons 
and/or exponentially large coefficients to approximate 
GMM discriminant functions
\end{itemize}

\subsection{Notations}
Throughput the paper, bold letter letters, such as $\xv$ and $\yv$, refer to vectors. Sets are denoted by calligraphic letters, such as $\Xc$ and $\Yc$. For a discrete set $\Xc$, $|\Xc|$ denotes its cardinality. ${\bf 0}_n$ denotes the all-zero vector in $\Real^n$. $I_n$ denotes the $n$-dimensional identity matrix. For $\xv\in\Real^n$, $\norm{\xv}^2=\sum_{i=1}^nx_i^2$.

\begin{figure}
  \includegraphics[width=\linewidth]{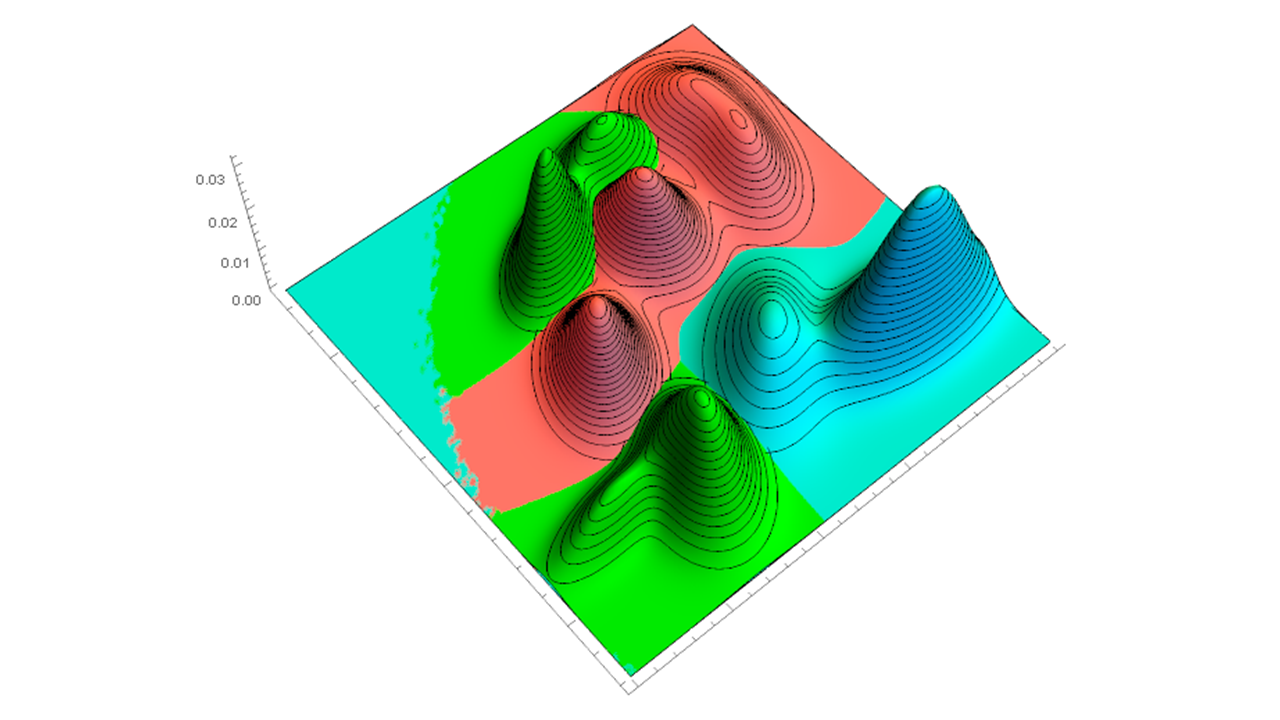}
  \caption{A collection of three GMMs with 4, 4 and 2 Gaussian components and their discriminant functions. Colors (red, green and blue) denote the three classes that the GMMs represent.}
  \label{fig:3GMMs}
\end{figure}
\subsection{$L$-layer neural networks and the activation function $\sigma$}\label{sec:l-layer}

\textit{$L$-layer Neural Network}. Consider a fully-connected neural network with $L$ hidden layers.  We refer to a network with $L=1$ hidden layer as an SNN and to a network with $L>1$ hidden layers as a DNN. Let $\xv\in\Real^n$ denote the input vector.   The function generated by an $L$-layer neural network, ${nn}:  \Real^{n} \rightarrow \Real^{c}$, where $c$ denotes the number of classes, can be represented as a composition of affine functions and non-linear functions $\sigma_l$, as follows
\begin{align}
{ nn}(\xv)=\sigma_{L+1} T^{[L+1]}\sigma_L T^{[L]}...\sigma_1 T^{[1]}(\xv).\label{eq:def-nn}
\end{align}
Here, $T^{[\ell]}:\Real^{n_{\ell-1}}\to \Real^{n_{\ell}}$ denotes the affine mapping applied at layer $\ell$.  Mapping  $T^{[\ell]}$ is represented by linear transformation $W^{[\ell]}\in\Real^{n_{\ell}\times n_{\ell-1}}$  and translation ${\bf b}^{[\ell]}$.  Moreover, $\sigma_{\ell}:\Real\to\Real$ denotes the non-linear activation function  
applied element-wise at layer $\ell$, $\ell=1,\ldots,L+1$. By convention, layer $L+1$ is referred to as the output layer. In this definition, $n_{\ell}$, $\ell=1,\ldots,L$, denotes the number of hidden  nodes in layer $\ell$. To make the notation consistent, for $\ell=0$, let $n_{0}=n$, and for $\ell=L+1$, $n_{L+1}=c$.
We will occasionally use the notations $dnn$ and $snn$ to signify that the cases of  $L>1$ and $L=1$, respectively.  In a classification task, the index of the highest value output  tuple determines the optimal class for input $\xv$. \\\\
\textit{The non-linear function $\sigma$}.  
For the two-layer construction in Section~\ref{sec:sufficiency}, we require some regularity
assumptions on the activation function $\sigma$, which are met by typical smooth NN activation
functions such as the sigmoid function. In Section~\ref{sec:remarks-conclusions}, 
we indicate how to refine the proofs to accommodate
the popular and simple $\rm ReLU$ activation function. The proof of the inefficiency of SNNs in
Section~\ref{sec:exp-nodes-necessary} applies to a very 
general class of activation functions including the $\rm ReLU$ function.


%
\subsection{GMMs and their optimal classification functions}\label{section:GMM-optimal}

Consider the problem of classifying points generated by a mixture of Gaussian distributions. 
We assume that there are $c$ classes and members of each class are drawn from a mixture of Gaussian 
distributions. Assume that there are overall $k$ different Gaussian 
distributions to draw from for all classes. Each Gaussian 
distribution is assigned uniquely to one of the $c$ classes. For $j\in\{1,\ldots,k\}$, let  
$\mvec{\mu}_{j}\in\Real^n$ and $\Sigma_{j}\in\Real^{n\times n}$ denote the 
mean and the covariance matrix of Gaussian distribution $j$. Also, let $\rho_i$, $i=1,\ldots,c$, 
denote the prior probability that a data point is drawn from Class $i$. Assume that the assignment 
of the Gaussian distributions to the classes is represented by sets $\Tc_1,\ldots,\Tc_c$, which 
form a partition of  $\{1,\ldots,k\}$. (That is, $\Tc_i\cap\Tc_j=\emptyset$, for $i\neq j$, 
and $\cup_{i=1}^c\Tc_i=\{1,\ldots,k\}$.) Set $\Tc_i$ represents the indices of the 
Gaussian distributions corresponding to Class $i$.
Finally, for Class $i$, let $w_j$, $j\in\Tc_i$, denote the conditional probability that within Class 
$i$, the data comes from Gaussian distribution $j$. By this definition, for $i=1,\ldots,n$, 
we have $\sum_{j\in\Tc_i}w_j=1.$
Under this model, and with a slight abuse of notation, the data are distributed as
\begin{align}
\sum_{i=1}^c \rho_i\sum_{j\in\Tc_i}w_{j} \Nc(\mvec{\mu}_{j},{\Sigma}_{j}).
\end{align}
Conditioned on being in Class $i$, the data points are drawn from a mixture  of $|\Tc_i|$ Gaussian distributions as $\sum_{j\in\Tc_i}w_{j} \Nc(\mvec{\mu}_{j},{\Sigma}_{j}).$
For $j=1,\ldots,k$, let $p^G_j:\Real^n\to\Real$ denote the probability density function (pdf) of a
single Gaussian distribution $\Nc(\mvec{\mu}_{j},{\Sigma}_{j})$. 

An optimal classifier $\Cc^*$,  maximizing  the  probability of membership across all classes, operates as follows
\begin{align}
\Cc^*(\xv) &= \argmax_{i\in \{1,\ldots,c\}}\; \rho_i \sum_{j\in\Tc_i}w_{j}p^G_{j}(\xv)\nonumber\\
 &= \argmax_{i\in \{1,\ldots,c\}}\; \rho_i \sum_{j\in\Tc_i}{w_{j}\over |2\pi \Sigma_{j}|^{1\over 2}}\exp\left(-{1\over 2}(\xv-\mvec{\mu}_{j})^T\Sigma_{j}^{-1}(\xv-\mvec{\mu}_{j})\right)\nonumber\\
  &= \argmax_{i\in \{1,\ldots,c\}}\; \rho_i \sum_{j\in\Tc_i}w_{j} \gamma_j \exp(-g_{j}(\xv) ), \nonumber
\end{align}
where, for $j=1,\ldots,k$,
\[
g_{j}(\xv)\triangleq{1\over 2}\xv^T\Sigma_{j}^{-1}\xv-\mvec{\mu}_{j}^T\Sigma_{j}^{-1}\xv+{1\over 2}\mvec{\mu}_{j}^T\Sigma_{j}^{-1}\mvec{\mu}_{j},
\]
and
\[
\gamma_{j}\triangleq |2 \pi \Sigma_{j}|^{-1/2}.
\]
  We  refer to  $\Cc^*(\xv)$ as the optimal classifier for the $c$ GMMs. 
Define the $i$-th discriminant function $d_i:\Real^n\to\Real$, as 
\begin{align}
d_i(\xv)  \triangleq \rho_i \sum_{j\in\Tc_i}w_{j}\gamma_j \exp(-g_{j}(\xv)). \label{eq:def-discriminant-func}
\end{align}
Note that these are the functions that we wish to approximate using 
DNNs and SNNs.  
Using this definition, the optimal classifier  $\Cc^*$ can be characterized in terms of the $c$ discriminant functions as 
\begin{equation}\label{eqn:optimal_classifier}
\Cc^*(\xv)=\argmax_{i\in\{1,\ldots,c\}} d_i(\xv).
\end{equation}
%


\section{Connection between classification and approximation}\label{sec:connection}

The main result of this paper is that the  discriminant functions described in \eqref{eq:def-discriminant-func},  required for computing optimal classification  
function $\Cc^*(\xv)$, can be approximated accurately by a relatively small  neural network with \textit{two} hidden layers, but that 
accurate approximation with a  \textit{single}  hidden layer network is only possible if either the  number of the nodes or the magnitudes of the coefficients  are exponentially large in $n$.


In this section, we first establish a connection between the accuracy in  approximating the  discriminant functions and the error performance of a classifier that is based on these approximations. Given a non-negative function $d(\xv)$,  $d:\Real^n\to\Real$, and threshold $t>0$,  let  $\Sc_{d,t}$ denote the superlevel set of  $d(\xv)$ defined as
\[
\Sc_{d,t} \triangleq \left\{ \xv\in\Real^n \, : \, d(\xv) \geq t \right\}.
\]
\begin{definition}
A function $\hat{d}:\Real^n\to\Real$  is a $(\delta,q)$-approximation of a non-negative function  $d:\Real^n\to\Real$ under a pdf $p$, if
there is a threshold $t$, such that  $\P_p [ \Sc_{d,t} ] \geq 1-q$, and 
\begin{eqnarray}
\left| \hat{d}(\xv)-d(\xv)\right|  \leq \delta d(\xv),  & & \xv \in \Sc_{d,t}\label{eq:cond1}\\
0 \leq \hat{d}(\xv) \leq (1+\delta) t, & & \xv \not\in \Sc_{d,t}.\label{eq:cond2}
\end{eqnarray}
Let $t_{\hat{d},\delta,q}$ denote the corresponding threshold. If there are multiple such thresholds, let  $t_{\hat{d},\delta,q}$ denote the infimum of all such thresholds.
\end{definition}
In this definition, $\hat{d}$ closely approximates $d$ in a relative sense, wherever  $d(\xv)$ exceeds threshold $t$. The function $\hat{d}$  is small (in an absolute sense), where $d(\xv)$ is small, an event that occurs with low probability under $p$. Although $p$ and $d$ need not be related in this definition, we will typically use it in cases where $d$ is just a scaled version of $p$.

Given two equiprobable classes with pdf functions $p_1$ and $p_2$, the optimal Bayesian classifier  chooses class $1$, if $p_1(\xv) > p_2(\xv)$, and class 2  otherwise. Let  $e_{21,{\rm opt}} = \P_1[p_2(\xv) > p_1(\xv)]$ denote  the probability of incorrectly deciding class 2, when the true distribution is class 1. If we classify  using {\em approximate} pdfs  with relative errors bounded by $\alpha \geq 1$, then the probability of error increases to
$e_{21,{\rm opt}}[\alpha] := \P_1[ p_2(\xv) > p_1(\xv)/\alpha ]$. Under appropriate conditions, $e_{\rm {\rm opt}}(\alpha)$ approaches $e_{\rm {\rm opt}}$, as $\alpha$ converges  to 1. Lemma \ref{lemma:approx-class} below  shows that   $(\delta,q)$-approximations of $p_1$ and $p_2$  enable us to approach $e_{21,{\rm opt}}$, by taking $\delta$ and $q$ sufficiently small.
\begin{lemma}\label{lemma:approx-class}
Given pdfs $p_1$ and $p_2$, let  $\hat{d}_1$ and $\hat{d}_2$ denote $(\delta,q)$-approximations of discriminant functions $d_1 = p_1$ and $d_2 = p_2$ under distributions $p_1$ and $p_2$, respectively. Define $t_i$, $i=1,2$, as $t_i\triangleq t_{\hat{d}_i,\delta,q}$. Consider  a classifier that declares class 1 when $\hat{d}_1(X) > \hat{d}_2(X)$ and class 2 otherwise. Then, the probability of error of this classifier, under distribution 1, is bounded by
$$
e_{21} \leq e_{21,{\rm opt}}[{1+\delta\over 1-\delta}] + q +\P_1(\Sc_{d_1,(1+\delta)t_2/(1-\delta)}).
$$
\end{lemma}
The proof of Lemma \ref{lemma:approx-class} is presented in Appendix \ref{sec:proof-lemma2}.

Note that as $q$ converges to zero, both $t_1$ and $t_2$ converge to zero as well. Therefore, letting $q$  converge to zero ensures that $\P_1((1+\delta)t_2\geq (1-\delta){d}_1(\xv))$ also converges to zero. 
One way to construct a nearly optimal classifier for two distributions is thus to independently build a $(\delta,q)$ approximation for each distribution, and then define the classifier based on maximum of the two functions.

With this motivation, in the rest of the paper, we focus on  approximating the discriminant functions defined earlier for classifying GMMs. In the next section, we show that using a two hidden-layer neural network, we can construct a $(\delta,q)$-approximation $\hat{d}$  of the discriminant function  of a GMM, see \eqref{eq:def-discriminant-func}, with input dimension $n$, with $O(n)$ 
nodes or neurons, for any $\delta,q>0$.

In the subsequent section, we show by contrast that even for the simplest GMM consisting of a single Gaussian distribution, even a weaker approximation that bounds the {\em expected} $\ell_2$  error   cannot  be achieved  by a single hidden-layer network, unless the (neural) network has  either exponentially many nodes or exponentially large coefficients. The weaker definition of approximation that we will use in the converse result is the following.
\begin{definition}\label{def:ell2}
A function $\hat{d}:\Real^n\to\Real$  is an $\epsilon$-relative ${\ell_2}$ approximation for a function  $d:\Real^n\to\Real$ under pdf $p$, if
$$
\E_p \left[ \left(\hat{d}(\xv)-d(\xv)\right)^2 \right] \leq \epsilon \E_p \left[ \left(d(\xv)\right)^2 \right].
$$
\end{definition}
The following lemma shows that if approximation under this weaker notion is not possible, it also is not possible under the stronger $(\delta,q)$ notion. 
\begin{lemma}\label{lemma3}
If $\hat{d}$ is a $(\delta,q)$-approximation of a distribution $d$ under distribution $p$, then it is also an $\epsilon$-relative $\ell_2$ approximation of $d$, with parameter
$$\epsilon = \delta^2  + \frac{(1+\delta)^2 q}{1-q}. $$
\end{lemma}
Proof of Lemma \ref{lemma3} is presented in Appendix~\ref{sec:proof-lemma3}.

\section{Sufficiency of two hidden-layer NN with $O(n)$ nodes}\label{sec:sufficiency}

\subsection{Overview}\label{subsec:overview}

We are interested in approximating the discriminant functions corresponding to optimal classification of  GMM, as defined in Section~\ref{section:GMM-optimal}. In particular, focusing on a particular class and dropping the subscript from $d_i(\xv)$, the function of interest is of the form
\begin{equation}
d(\xv) = \sum_{j=1}^J \beta_j \exp \left( -g_j(\xv) \right)  \label{eqn:discriminant}
\end{equation}
where
$g_j(\xv) = \frac{1}{2} (\xv-\mvec{\mu}_j)^T \Sigma_j^{-1} (\xv-\mvec{\mu}_j) = \frac{1}{2} (\xv-\mvec{\mu}_j)^T \left(\Sigma_j^{-1/2}\right)^T \left(\Sigma_j^{-1/2}\right)(\xv-\mvec{\mu}_j)$, where  $J$ is the number of Gaussian distributions in the mixture, and where $\beta_j = \rho w_j \gamma_j = \rho w_j |2\pi \Sigma_j|^{-1/2}$ for fixed prior $\rho$ and conditional probabilities $w_j$. Here, $\Sigma_j^{-1/2}$ denotes any symmetric decomposition of $\Sigma_j^{-1}$, such as the Cholesky decomposition.

We first observe that the function $g_j(\xv)$ is a general quadratic form in $\Real^{n}$
and thus consists of the sum of $O(n^2)$ product terms of the form $x_ix_j$.  Since each such product term can be approximated  via  four neurons (see \cite{lin2017does}), $g_j(\xv)$ can be 
approximated arbitrarily well using $O(n^2)$  neurons.   We can however reduce the number of nodes further by applying the affine transformation $\yv_j = \Sigma_j^{-1/2} (\xv-\mvec{\mu}_j)$ in the first layer, so that $g_j(\xv)$ is simply $\norm{\yv_j}^2=\sum_{i=1}^ny_{j,i}^2$, i.e., it consists of $n$ quadratic terms, and  can in turn can be   approximated via  $O(n)$  nodes. This $O(n^2)$ to $O(n)$ reduction in the number of neurons is specific to quadratic polynomials
which are generic exponents of GMM discriminant functions and  we will take advantage of this reduction in our proofs.

The key GMM discriminant function, which is a sum of exponentials of quadratic forms, has a natural implementation 
as a neural network. To get insight into the construction we use to prove the main result, consider the the NN model (\ref{eq:def-nn}) with $L=2$, 
$\sigma_1(x) = x^2$, and $\sigma_2(x)=\exp(-x)$. 
To compute the $j$-th component of $d(\xv)$ (\ref{eqn:discriminant}),
\begin{itemize}
\item Start with input $\xv\in\Real^n$ via $n$ input nodes.
\item Apply linear weights and biases to obtain $n$ results in the vector $\yv^{(j)} = \Sigma_j^{-1/2}\xv - \Sigma_j^{-1/2}\mvec{\mu}_j$.
\item Apply the activation function $\sigma_1(x)=x^2$ componentwise to $\yv^{(j)}$ to obtain $n$-vector $\zv^{(j)}$, the outputs of the first hidden layer
\item Apply $n$  unit weights to form the single sum $g_j(\xv) = \sum_k z_k^{(j)}$.
\item Apply the activation function $\sigma_2(x) = e^{-x}$ to obtain $\exp(-g_j(\xv))$, an output of the second hidden layer.
\end{itemize}
The overall network contains $J$ subnetworks, with the $j$-th subnetwork calculating $\exp(-g_j(\xv))$. The final output layer sums each of these outputs with weight $\beta_j$ to obtain $d(\xv)$.
The network contains $nJ$ nodes in the first hidden layer and $J$ in the second hidden layer, and calculates $d(\xv)$ exactly for any input $\xv$.

In many settings, one is interested in the expressibility of neural networks when restrictions are placed on the allowable activation functions. For example, we may focus on the case that a single common activation function must be used in each node, or that the activation functions have specific regularity properties. In such cases, 
we can modify our network, replacing nodes implementing $x^2$ and $\exp(-x)$ with groups of nodes implementing a more basic activation function, whose outputs are summed to achieve approximately the same result. We refer to such a group of basic nodes as a super-node.

In the proof, we rely on various assumptions on the activation function $\sigma(x)$. All of the assumptions are satisfied by the sigmoid function $\sigma(x) = 1/(1+\exp(-x))$, for example.

\begin{assumption}[Curvature]\label{assum:2nd_derivative}
There is a point $\tau\in\Real$, and parameters $r$ and $M$, such that
$\sigma^{(2)}(\tau)>0$, such that
  $\sigma^{(3)}(x)$ exists and is bounded, $|\sigma^{(3)}(x)| \leq M$,  in the neighborhood $\tau-r\leq x \leq \tau+r$.
  \end{assumption}
\begin{assumption}[Monotonicity]\label{assum:monotonic_sum}
The symmetric function $\sigma(x+\tau)+\sigma(-x+\tau)$ is monotonically increasing for $x\geq 0$, with  $\tau$ as defined in Assumption \ref{assum:2nd_derivative}.
\end{assumption}
\begin{assumption}[Exponential Decay]\label{assum:exp_decay}
There is $\eta > 0$ such that $|\sigma(x)|\leq \exp(\eta x)$ and 
$|\sigma(x)-1| \leq \exp(-\eta x)$.
\end{assumption}

Assumptions~\ref{assum:2nd_derivative} and \ref{assum:monotonic_sum} can be  used to construct an approximation of $x^2$ using $O(1)$ basic nodes for each such term. They are satisfied, for example, by common activation functions such as the sigmoid and 
$\tanh$ functions. Assumption~\ref{assum:exp_decay} is used to construct an approximation of $\exp(-x)$ in the second hidden layer, with  $O(n)$ nodes in each of $J$ subnetworks. This assumption is met by the indicator function $u(x) = 1\{x>0\}$, and any number of activation functions that are smoother versions of $u(x)$, including piecewise linear approximations of $u(x)$ constructed with $\rm ReLU$, and the sigmoid function.

The following is our main positive result, about the ability to efficiently approximate a GMM discriminant function with a two-layer neural network.

\begin{theorem-new}\label{theor:main_two_layer}
Consider a GMM with discriminant function $d:\Real^n\to \mathds{R}^{+}$ of the form (\ref{eqn:discriminant}), consisting of Gaussian pdfs with bounded covariance matrices. Let the activation function $\sigma:\Real \to \Real$  satisfy Assumptions ~\ref{assum:2nd_derivative}, \ref{assum:monotonic_sum}, and  \ref{assum:exp_decay} . Then for any given $\delta >0$ and $q \in (0, 1)$, there exists a two-hidden-layer neural network consisting of $ M = O(n)$ instances of the activation function $\sigma$ such that  its output function $\hat{d}$  is a $(\delta,q)$ approximation of $d$. 
\end{theorem-new}

\begin{remark}
Applying Theorem~\ref{theor:main_two_layer}  to a collection of $c$ GMMs gives rise to a DNN with $O(n)$ neurons that approximates the optimal classifier of these GMMs via (\ref{eqn:optimal_classifier}).
\end{remark}

\begin{remark}
The construction of an $O(n)$-neuron approximator of the GMM discriminant function assumes that the eigenvalues of the covariance matrices are bounded from above and also bounded away from zero.
\end{remark}

The detailed proof of Theorem \ref{theor:main_two_layer} is presented in Appendix \ref{app:proof-thm-1}. The proof relies on several lemmas that are stated and proved in Appendix~\ref{sec:supporting-lemmas}.
The main steps of the proof are the following:
\begin{itemize}
\item Map the overall approximation requirements to a related requirement that applies to each of the $J$ Gaussian components. Thereafter, we focus on a sub-network implementing a single Gaussian component.
\item Show that the first hidden layer can be constructed from $n$ pairs of nodes, where each pair of nodes  forms a sufficiently accurate approximation of $x^2$.
\item Show that the  second hidden layer can be constructed from a set of basic nodes that combine to form a sufficiently accurate approximation of $\exp(-x)$.
\item Show that the composition of two layers as constructed yields the required approximation accuracy.
\item Show that the number of basic nodes in the second hidden layer is $O(n)$.
\end{itemize}

As a first step in the proof of Theorem \ref{theor:main_two_layer}, given $d(\xv) = \sum_{j=1}^J \beta_j \exp \left( -g_j(\xv) \right)$, we build a neural net consisting of  $J$ sub-networks, with sub-network $j$ approximating $\beta_j \exp \left( -g_j(\xv) \right)$.
For convenience, denote by $c_j(\xv) = \beta_j \exp(-g_j(\xv))$, the desired output of the $j$-th subnetwork. The $J$ subnetwork function approximations $\hat{c}_j(\xv)$, $j=1,\ldots,J$, are summed up to get the final output $\hat{d}(\xv) = \sum_j \hat{c}_j(\xv)$.

\begin{lemma}\label{lemma:decomposition}
Given $\delta>0$, $q>0$, and the GMM discriminant function $d(\xv)=\sum_{j=1}^J c_j(\xv)$, let $t^*$ be such that $\P[\Sc_{d,t^*}] \geq 1-q$ under pdf $p(\xv) = \rho^{-1}  d(\xv)$ .
Define $\lambda = (t^*\delta)/(2J(1+\delta))$, and for each $j$, suppose we have an approximation function $\hat{c}_j$ of $c_j$ such that
\begin{eqnarray*}
\left| \hat{c}_j(\xv)-c_j(\xv) \right| \leq \delta/2 c_j(\xv),&  \mbox{if} \; c_j(\xv) \geq \lambda,\\
0 \leq \hat{c}_j(\xv) \leq \lambda(1+\delta),& \mbox{otherwise} &
\end{eqnarray*}
Then $\hat{d}(\xv) = \sum_j \hat{c}_j(\xv)$ is a $(\delta,q)$-approximation of $d(\xv)$ under $p(\xv)$.
\end{lemma}

See Appendix~\ref{plemma:decomposition} for proof.

This lemma establishes a sufficient standard of accuracy that we will need for the subnetwork associated with each Gaussian component. In particular, there is a level $\lambda$ such that we need to have relative error better than $\delta/2$ when the component function is greater than $\lambda$. Where the component function is smaller than $\lambda$, we require only an upper bound on the approximation function. The critical level $\lambda$ is proportional to $t^*$, which is a level exceeded with high probability by the overall discriminant function $d$. The scaling of the level $t^*$ with $n$ is an important part of the proof of Theorem~\ref{theor:main_two_layer}, and is analyzed later in Lemma \ref{lemma11}.
See Appendix \ref{app:proof-thm-1} for the remaining steps in  the proof of Theorem \ref{theor:main_two_layer}.


\section{Exponential size of SNN for approximating the GMM discriminant function}\label{sec:exp-nodes-necessary}

In the previous section, we showed that a DNN with two hidden layers, and $O(n)$ hidden nodes is able to approximate the discriminant functions corresponding to an optimal Bayesian classifier for a collection of GMMs. In this section, we prove a converse result for SNNs. More precisely, we prove that for an SNN to approximate the discriminant function of even a single Gaussian distribution, the number of nodes needs to grow exponentially with $n$.  

Consider a  neural network with a single hidden layer consisting of $n_1$ nodes. As before, 
let  $\sigma:\mathds{R}\to\mathds{R}$ denote the  non-linear function applied by  each hidden 
neuron. For $i=1,\ldots,n_1$, let $\wv_i\in\mathds{R}^n$,  and 
$b_i\in\mathds{R}$ denote the weight vector and the bias corresponding to node $i$, 
respectively. The function generated  by this network can be written as
\begin{align}
f(\xv)&=\sum_{i=1}^{n_1} a_i\sigma(\langle \wv_i,\xv\rangle+b_i)+a_0.\label{eq:def-f-orig}
\end{align}
Suppose that $\xv\in\Real^n$ is distributed  as $\Nc(\mvec{0}_0,{s}_xI_n)$, with pdf  $\mu:\Real^n\to\Real$.  Suppose that the function to be approximated is 
\begin{align}
\mu_c(\xv)\triangleq \left({s}_f+2{s}_x \over {s}_f\right)^{n\over 4}\ex^{-{1\over 2{s}_f}\|\xv\|^2},\label{eq:def-mu-c}
\end{align}
 which has the form of a symmetric zero-mean Gaussian distribution with variance $s_f$ in each direction, and has been normalized so that $ \E[\mu_c^2(\xv)]=1$.  Our goal is to show that unless the number of nodes $n_1$ is exponentially large in the input dimension $n$, the  network cannot approximate the function $\mu_c$ defined in \eqref{eq:def-mu-c} in the sense of Definition~\ref{def:ell2}.

Our result applies to very general activation functions, and allows a different activation function in every node; essentially all we require is that i) the response of each activation function depends on its input $\xv$ through a scalar product $\langle \wv_i,\xv\rangle$, and ii) the output of each hidden node is square-integrable with respect to the Gaussian distribution $\mu(\xv)$. Incorporating the constant term $a_0$
into one of the activation functions, we consider a more general model
\begin{align}
f(\xv)&=\sum_{i=1}^{n_1} a_i h_i(\langle \wv_i,\xv\rangle)\label{eq:def-f-alt}
\end{align}
for a set of functions $h_i:\Real \to\Real$.
To avoid  scale ambiguities in the definition of the coefficients $a_i$ and the functions $h_i$, we scale $h_i$ as necessary so that, for $i=1,\ldots,n_1$, $\norm{\wv_i} = 1$ and  $\E [ (h_i(\langle \wv_i,\xv\rangle))^2 ] = 1$.

Our main result in this section shows the connection between the number of nodes ($n_1$) and the achievable approximation error $\E[|\mu_c(\xv)-f(\xv)|^2]$. We focus on on $\epsilon$-relative $l_2$ approximation, as defined in Definition \ref{def:ell2}, which, as shown in Lemma \ref{lemma3} is weaker than the notion used in Theorem \ref{theor:main_two_layer}. Therefore, proving the lower bound under the weaker notion, automatically proves the same bound under the stronger notion as well.  

\begin{theorem-new}\label{thm:main}
Consider $\mu_c:\mathds{R}^n\to\mathds{R}$ and $f:\mathds{R}^n\to\mathds{R}$ defined in \eqref{eq:def-mu-c} and \eqref{eq:def-f-alt}, respectively, for some $s_f>0$. 
Suppose that  random vector $\xv\sim  \Nc({\bf0}_n,{s}_xI_n)$, where $s_x>0$. For $i=1,\ldots,n_1$, assume that  $\norm{\wv_i} = 1$,  and  $\E \left[ \left(h_i(\langle \wv_i,\xv\rangle)\right)^2 \right] = 1$, for activation functions $h_i:\Real\to\Real$. Then, 
\begin{align}
\E[|\mu_c(\xv)-f(\xv)|^2] \geq 1-  2 \sqrt{n_1} \norm{\av}   \left(1 + s_x/s_f\right)^{1/4} \rho^{-n/4},
\end{align}
where 
\begin{align}
\rho \triangleq  1 + {s_x^2\over s_f^2 + 2 s_x s_f} > 1.\label{eq:def-rho}
\end{align}
\end{theorem-new}

This result shows that if we want to form an $\epsilon$-relative $\ell_2$ approximation of $\mu_c$, in the sense of Definition~\ref{def:ell2}, with an SNN, $n_1$ must satisfy
\begin{align}
n_1 \geq \frac{1-\epsilon}{2A \left(1+s_x/s_f\right)^{1/4}}\rho^{n/4},
\end{align}
where $A={1\over \sqrt{n_1}}\norm{\av}$ denotes the root mean-squared value of $\av$. 
That is, the number of nodes need to grow exponentially with $n$, unless the magnitude $A$ of the final layer coefficients vector $\norm{\av}$ grows exponentially in $n$ as well.
Note that in the natural case $s_f = s_x$ where the discriminant function to be approximated matches the distribution of the input data, the required exponential rate of growth is $\rho^{n/4} = (4/3)^{n/4}$.

\begin{proof}[Proof of Theorem~\ref{thm:main}]
The mean squared error can be bounded as
\begin{eqnarray*}
\E[|\mu_c(\xv)-f(\xv)|^2] & \geq &  \ E[|\mu_c(\xv)|^2] - 2 \E\left[ \mu_c(\xv)f(\xv)\right]  + \ E\left[|f(\xv)|^2\right] \\
  & \geq & 1 - 2 \sum_i a_i \E \left[ \mu_c(\xv) h_i(\langle \wv_i, \xv\rangle) \right] \\
 & \geq & 1 - 2 \norm{\av} \norm{\rv}
\end{eqnarray*}
where the last step follows from the Cauchy-Schwarz inequality. Here, the $\rv\in\Real^{n_1}$, and $r_i \triangleq \E \left[ \mu_c(\xv) h_i(\langle \wv_i, \xv\rangle) \right]$.

We next bound the magnitude of $r_i$, showing that each is exponentially small in $n$.
By the rotational  symmetry of $\mu_c$ and  $\mu$, we have
\[
 \E \left[ \mu_c(\xv) h_i(\langle \wv_i, \xv\rangle) \right]=\E[\mu_c(\xv) h_i(x_1)].
\]
Defining $\alpha = 1 + 2s_x/s_f$ and $\beta = \left( 1/s_x + 1/s_f\right)$, we write
\begin{eqnarray*}
r_i & = & \alpha^{n/4} \left(2\pi s_x\right)^{-n/2} \int h(x_1) e^{-\frac{1}{2\beta} \normsq{\xv}} d\xv \\
 & = & \alpha^{n/4} \left(2\pi s_x\right)^{-n/2} \left(2\pi \beta\right)^{(n-1)/2} \int h(x_1) e^{-\frac{1}{2\beta} x_1^2}  dx_1\\
 & = & \alpha^{n/4} \left(\beta/ s_x\right)^{n/2}  \int h(x_1) e^{-\frac{1}{2\beta} x_1^2} \left( 2\pi \beta\right)^{-1/2} dx_1
\end{eqnarray*}
Note that $ \int h(x_1) e^{-\frac{1}{2\beta} x_1^2} \left( 2\pi \beta\right)^{-1/2} dx_1$ is the expected value of $ h(x_1)$, with respect to $x_1\sim\Nc(0,\beta)$. Therefore, using the Jensen inequality, 
\[
\left( \int h(x_1) e^{-\frac{1}{2\beta} x_1^2} \left( 2\pi \beta\right)^{-1/2} dx_1 \right)^2\leq  \int (h(x_1))^2 e^{-\frac{1}{2\beta} x_1^2} \left( 2\pi \beta\right)^{-1/2} dx_1.
\]
Therefore, 
\begin{eqnarray*}
r_i^2 & \leq & \alpha^{n/2} \left(\beta/ s_x\right)^{n}  \int\left( h(x_1)\right)^2 e^{-\frac{1}{2\beta} x_1^2} \left( 2\pi \beta\right)^{-1/2} dx_1 \\
& \leq & \alpha^{n/2} \left(\beta/ s_x\right)^{n-1/2}  \int\left( h(x_1)\right)^2 e^{-\frac{1}{2 s_x} x_1^2} \left( 2\pi s_x\right)^{-1/2} dx_1 \\
& \leq & \alpha^{n/2} \left(\beta/ s_x\right)^{n-1/2}
\end{eqnarray*}
where  the second step holds because $\beta = s_x s_f / (s_f + s_x) < s_x$, and therefore,  $\exp(-x_1^2/(2\beta) ) < \exp(-x_1^2/(2 s_x) ) $, and  the last step follows from our initial assumption that $\E [ (h_i(x_1))^2] = 1$.
Noting that
\begin{align*}
\alpha^{n/4} \left(\beta/s_x\right)^{n/2} = \left[ \left( \frac{s_f+2 s_x}{s_f}\right)\left( \frac{s_f}{s_f+s_x}\right)^2 \right]^{n/4}
 = \left( \frac{sf^2 + 2 s_f s_x + s_x^2}{s_f^2 + 2 s_f s_x}\right)^{-n/4} = \rho^{-n/4},
\end{align*}
we obtain
\begin{align}
|r_i| \leq \rho^{-n/4} (1+s_x/s_f)^{1/4}
\end{align}
This establishes that $\norm{\rv} \leq \sqrt{n_1} \rho^{-n/4} (1+s_x/s_f)^{1/4}$, which finishes the proof.
\end{proof}
 \begin{remark}
The  generalized model \eqref{eq:def-f-alt} covers a large class of  activation functions. It is straightforward to confirm that the required conditions are satisfied by bounded activation functions, such as the sigmoid  function or the $\tanh$ function, with arbitrary bias values. For the popular ReLU function,  $h_i(\langle \wv_i,\xv\rangle)=\max(|\langle \wv_i,\xv\rangle+b_i|,0)$. Therefore, $\E[|h_i(\langle \wv_i,\xv\rangle)|^2]\leq \E[(\langle \wv_i,\xv\rangle+b_i)^2]=s_x+b_i^2$, which again confirms the desired square-integrability  property. 
\end{remark}

\begin{remark}
From the point of numerical stability, it is natural to require  the norm of the final layer coefficients, $\norm{\av}$,  to be bounded, as the following simple argument shows. Suppose that network implementation  can compute each activation function $h_i(x)$ exactly, but that the implementation represents each coefficient $a_i$ in a floating point format with a finite precision. To gain intuition on the effect of this quantization noise, consider the following modeling. The implementation replaces  $a_i$ with $a_i + z_i$,  where  $\E[z_i]=0$ and $\E[z_i^2]=\nu |a_i|^2$ and $\nu$ reflects the level of precision in the representation.  Further assume that $z_1,\ldots,z_{n_1}$ are independent of each other and of  $\xv$. Then, the  error due to quantization can be written as
\begin{align}
\E \left[ \sum_i z_i^2 \left(h_i(\langle \wv_i,\xv\rangle)\right)^2 \right] = \sum_i \nu |a_i|^2 \E \left[\left(h_i(\langle \wv_i,\xv\rangle)\right)^2\right] = \nu \norm{\av}^2
\end{align}
In such an implementation, in order to keep the quantization error significantly below the targeted overal error $\epsilon$, we  need to have
\begin{align}
\norm{\av} \ll \sqrt{\epsilon/\nu}.
\end{align}
Unless the magnitudes of the weights used in the output layer are bounded in this way, accurate computation is not achievable in practice.
\end{remark}

\section{Sufficiency of exponentially many neurons}

In Section~\ref{sec:exp-nodes-necessary}, we studied the ability of an SNN  to approximate function $\mu_c$ defined as \eqref{eq:def-mu-c} and showed that such a network, if the weights are not allowed to grow exponentially with $n$, requires exponentially many nodes to make the error small. Clearly, Theorem \ref{thm:main} is a converse result, which  implies that the number of nodes $n_1$ should grow with $n$, at least as $\rho^{n\over 4}$ ($\rho>1$). The next natural question is the following: Would exponentially many nodes actually suffice in order to approximate function $\mu_c$?  In this section, we answer this question affirmatively and show a simple construction with random weights that, given enough nodes,  is able to well approximate function $\mu_c$ defined in \eqref{eq:def-mu-c}, within the desired accuracy. Recall that $\mu_c(\xv)= \alpha^{n\over 4}\exp(-{1\over 2{s}_f}\|\xv\|^2)$, 
where $\alpha\triangleq {{s}_f+2{s}_x \over {s}_f}.$
Consider  the output function of a single-hidden layer  neural network with all biases set to zero. The function generated by such a network  can be written as
\begin{align}
f(\xv)&=\sum_{i=1}^{n_1} a_i\sigma(\langle \wv_i,\xv\rangle).\label{eq:func-f-no-bias}
\end{align}
As before, here, $\sigma:\mathds{R}\to\mathds{R}$ denotes the non-linear function and $\wv_i\in\mathds{R}^n$, $\|\wv_i\|=1$,  denotes the weights used by hidden node $i$. To show sufficiency of exponentially many nodes, we consider a particular non-linear function $\sigma(x)=\cos({x\over \sqrt{{s}_f}})$.

\begin{theorem-new}\label{thm:random-weights}
Consider function $\mu_c:\mathds{R}^n\to\mathds{R}$, defined in \eqref{eq:def-mu-c}. Also, consider  random vector $\xv\in\mathds{R}^n$,  where $x_1,\ldots,x_n$ i.i.d.~$\Nc(0,{s}_x)$. Consider function $f:\mathds{R}^n\to\mathds{R}$ defined in \eqref{eq:func-f-no-bias} and assume that, for $i=1,\ldots,n_1$, $a_i={\alpha^{n\over 4} \over n_1}$, where $\alpha= 1+2{{s}_x / {s}_f}$,  and that
\begin{align}
\sigma(x)=\cos({x\over \sqrt{{s}_f}}).
\end{align}
Given $\epsilon>0$, assume that
\[
n_1>{1\over \epsilon} \alpha^{n\over 2}.
\]
Then, there exists weights $\wv_1,\ldots,\wv_{n_1}$ such that
\[
\E_{\xv}[(f(\xv)-\mu_c(\xv))^2] \leq \epsilon.
\]
\end{theorem-new}
Theorem \ref{thm:random-weights} is proved in 
Appendix \ref{appendix-last}.

To better understand the implications of Theorem \ref{thm:random-weights} and how it compares against Theorem \ref{thm:main}, define
\[
m_1={ \rho}=1 +  {{s}_x\over 2{s}_f}\Big(1-{s_f\over s_f + 2 s_x }\Big),
\]
and
\[
m_2=\alpha^2=1+ {4{s}_x\over {s}_f}\Big(1+{{s}_x\over {s}_f}\Big).
\]
where $\rho$ is defined in \eqref{eq:def-rho}. It is straightforward to see that $1< m_1< m_2$, for all positive values of $({s}_x,{s}_f)$.
Theorems \ref{thm:main} and  \ref{thm:random-weights} show that there exist constants $c_1$ and $c_2$, such that if the number of hidden nodes in a single-hidden-layer network ($n_1$) is smaller than $c_1m_1^{n\over 4}$, the expected  error in approximating function $\mu_c(\xv)$ must get  arbitrarily  close to one.
On other hand, if $n_1$ is larger than $c_2m_2^{n\over 4}$, then there exists a set of weights such that the error can be made arbitrary close to zero.  In other words, it seems that there is a phase transition in the exponential growth of the number of nodes, below which, the function cannot be approximated with a single hidden layer. 
Characterizing that phase transition more precisely
is an interesting open question, which we leave to future work.

\section{Related work}
There is a rich and well-developed literature on the complexity of Boolean 
circuits, and the important role depth plays in them. 
However, since it is not clear to what extend a result on Boolean circuits has a consequence for DNNs,  we
do not summarize this literature.  The interested reader may wish to start with 
\cite{boolean-survey}.  A key notion for us is that of depth, namely, 
the number of (hidden) layers of neurons in a neural network as defined
in Section~\ref{sec:l-layer}.  We are interested to know to
what extent, if any, depth reduces complexity of the neural network to 
express or approximate functions of interest in classification.  It is not
the complexity of the function that we want to approximate that matters,
because the UAT already tells us that regular functions, which include discriminant 
functions we discussed in Section~\ref{section:GMM-optimal}, can be approximated by SNNs, shallow neural networks.  But the complexity of the NNs, as measured by the \textit{number} of neurons needed for the approximation is of interest to us.   In this respect, the work of  \cite{delalleau2011shallow,martens2014expressive,cohenetal} contain approximation 
results for neural structures for certain polynomials and tensor functions,
in the spirit of what we are looking for, but as with Boolean circuits, these 
models deviate substantially from the standard DNN models we consider 
here, those that represent the neural networks that have worked 
well in practice and for whose behavior we wish to obtain  fundamental insights.
 
Remarkably, there is a small collection of recent results which, 
as in this report, show that adding a single layer to an 
SNN reduces the number of neurons by a logarithmic factor 
for approximation of some special functions: 
see \cite{telgarsky,eldan2016power,rolnick2018the} 
for approximation of high-degree polynomials, a certain radial function, and 
saw-tooth functions, respectively. Our work is therefore in the same 
spirit as these, showing \textit{the power of two} in the reduction of complexity of 
DNNs, and is therefore, the continuation and generalization of the said 
set of results and is especially informed by \cite{eldan2016power}.

\section{Remarks and conclusion}\label{sec:remarks-conclusions}

Two remarks are worth making at the conclusion of this paper.  First,
even though we used a variety of sufficient 
regularity assumptions for 
the non-linear function $\sigma$ used in the construction of our DNNs,
these assumptions are not necessary to construct an efficient two-layer network.
For example, to construct a network using the commonly used Rectifier Linear 
Unit ($\rm ReLu$) activation, in the first layer we can form $n$ 
super-nodes, each of which has a piecewise constant response $h_i(x)$ that approximates $x^2$ with the accuracy specified in Lemma~\ref{lemma:first_layer}. The number of basic nodes needed in each super-node in this construction is $2R/\sqrt{\nu}$, where $R$ and $\nu$ denote the range and the accuracy for approximating $x^2$ in layer one, respectively. The analysis of  $R$ and $\nu$ in Lemma~\ref{lemma:composed_accuracy} and \ref{lemma:tstar} shows that $R$ is $O(\sqrt{n})$ and $\nu$ is $O(1/n)$, so that the number of nodes needed per super-node in the first layer is now $O(n)$, compared to $O(1)$ in the construction presented in Section~\ref{sec:sufficiency}. Since there are $n$ such nodes, the total number of basic nodes in the network becomes  $O(n^2)$ - still an exponential reduction compared with a single layer network.

The second remark is that in this work we assumed the distribution functions of all 
the GMMs are given and data is tightly represented in $\Real^{n}$, that is,
the covariance matrices are not degenerate.   
Of course in practice no distributions are given (GMMs or others) 
and there is often much redundancy and the ambient dimension of data could be reduced
significantly.  Many layers of practical DNNs are essentially used to learn the 
empirical distributions of the data and possibly reduce redundant 
dimensions via projections, albeit implicitly.  Therefore, our construction 
is meant to be prototypical and 
demonstrative of the inherent value of depth for reduction of complexity 
rather than a prescription for implementation.\\\\
%

\section{Appendices}
\appendix

\section{Proof of Lemma \ref{lemma:approx-class}}\label{sec:proof-lemma2}
Note that 
\begin{align}
e_{21}&=\P_1(\hat{d}_2(\xv)\geq \hat{d}_1(\xv))=\P_1(\hat{p}_2(\xv)\geq \hat{p}_1(\xv),\xv\in\Sc_{p_1,t_1})+\P_1(\hat{p}_2(\xv)\geq \hat{p}_1(\xv)),\xv\in\Sc^c_{p_1,t_1})\nonumber\\
&\leq \P_1(\Sc^c_{p_1,t_1}) +\P_1(\hat{p}_2(\xv)\geq \hat{p}_1(\xv),\xv\in\Sc_{p_1,t_1})\nonumber\\
&\leq q+\P_1(\hat{p}_2(\xv)\geq \hat{p}_1(\xv),\xv\in\Sc_{p_1,t_1},\xv\in\Sc_{p_2,t_2})+\P_1(\hat{p}_2(\xv)\geq \hat{p}_1(\xv),\xv\in\Sc_{p_1,t_1},\xv\in\Sc^c_{p_2,t_2})\nonumber\\
&\leq q+\P_1((1+\delta)p_2(\xv)\geq (1-\delta)p_1(\xv))+\P_1(\hat{p}_2(\xv)\geq \hat{p}_1(\xv),\xv\in\Sc_{p_1,t_1},\xv\in\Sc^c_{p_2,t_2})\nonumber\\
&\leq q+\P_1({1+\delta\over 1-\delta} p_2(\xv)\geq p_1(\xv))+\P_1((1+\delta)t_2\geq (1-\delta){p}_1(\xv)),
\end{align}
which yields the desired result. 


\section{Proof of Lemma \ref{lemma3}}\label{sec:proof-lemma3}
Let $t$ be the threshold in the definition of $(\delta,q)$ approximation. Where $d$ exceeds $t$, we have $\E [ (\hat{d}-d)^2 \,|\, \Sc_{d,t}] \leq \delta^2 \E [ d^2 \,|\, \Sc_{d,t}]$, and where $d$ is less than $t$, we have $\E [ (\hat{d}-d)^2 \,|\, \bar{S}_{d,t}] \leq (1+\delta)^2t^2$.
We also know that
$$
\E [d^2] \geq \E[d^2 \,|\, \Sc_{d,t}]\P[\Sc_{d,t}] \geq t^2(1-q)
$$
Hence, we have
\begin{eqnarray*}
\E [ (\hat{d}-d)^2] & = & \E [ (\hat{d}-d)^2 \,|\, \Sc_{d,t}]\P[\Sc_{d,t}] + \E [ (\hat{d}-d)^2 \,|\, \bar{S}_{d,t}]\P[\bar{S}_{d,t}] \\
& \leq & \delta^2 \E [ d^2 \,|\, \Sc_{d,t}]\P[\Sc_{d,t}] + (1+\delta)^2 t^2 q \\
& \leq & \delta^2 \E[d^2] + \frac{(1+\delta)^2 q}{1-q} \E[d^2].
\end{eqnarray*}

\section{Proof of Lemma~\ref{lemma:decomposition}} \label{plemma:decomposition}

First, suppose that $x\in \Sc_{d,t^*}$.
We have
\begin{eqnarray*}
|\hat{d}(\xv)-d(\xv)| & \leq & \sum_j |\hat{c}_j(\xv)-c_j(\xv)| \\
& \leq & \sum_{j: c_j(\xv) \geq \lambda} |\hat{c}_j(\xv)-c_j(\xv)| + \sum_{j: c_j(\xv) < \lambda} |\hat{c}_j(\xv)-c_j(\xv)|
\end{eqnarray*}
Consider the first partial sum, over $j$ such that $c_j(\xv) \geq \lambda$. By assumption each term in the sum is bounded by $\delta/2 c_j(\xv)$, and so the partial sum is upper bounded by $\delta/2 d(\xv)$.
Considering the second partial sum, since $0\leq c_j(\xv)<\lambda$ and $0\leq \hat{c}_j(\xv) \leq \lambda(1+\delta)$, each term is upper bounded by $\lambda(1+\delta)$, and the partial sum upper-bounded by $J\lambda(1+\delta) = t^*\delta/2$.
Since $d(\xv) \geq t^*$, the sum is upperbounded by $\delta/2 d(\xv)$.
Putting both sums together, $|\hat{d}(\xv)-d(\xv)| \leq \delta d(\xv)$.
Thus $\hat{d}$ has the required relative accuracy on $\Sc_{d,t^*}$.

Now, suppose $\xv\not\in \Sc_{d,t^*}$, i.e. $t^* > d(\xv) \geq c_j(\xv)$. If $c_j(\xv) < \lambda$, then $\hat{c}_j(\xv) \leq \lambda(1+\delta) \leq t^*\delta/(2J)$.  If $c_j(\xv) \geq \lambda$, then $\hat{c}_j(\xv) \leq (1+\delta/2) c_j(\xv)$. Putting both together, we have
\begin{eqnarray*}
\hat{c}_j(\xv) & \leq & (1+\delta/2)c_j(\xv) + \frac{t^*\delta}{2J} \\
\sum_j \hat{c}_j(\xv) & \leq & (1+\delta/2)d(\xv) + \frac{t^*\delta}{2} \\
 & \leq & (1+\delta)t^*
\end{eqnarray*}

\section{Lemmas supporting Theorem~\ref{theor:main_two_layer}}\label{sec:supporting-lemmas}

In this section, we state and prove some technical lemmas used in the proof of Theorem \ref{theor:main_two_layer}.

\begin{lemma}\label{lemma:first_layer}
Given range $R>0$ and desire accuracy $\nu > 0$, let $h(x,a)$, defined in \eqref{eq:def-h}, be constructed from an activation function $\sigma$ satisfying Assumptions~\ref{assum:2nd_derivative} and \ref{assum:monotonic_sum}, with $r$, $M$, and $\tau$ as defined in those assumptions.
Choose $a$ satisfying $a \geq 2R/r$, $a \geq 8MR^3/(3\nu \sigma^{(2)}(\tau))$ and $a \geq 4MR/(3\sigma^{(2)}(\tau))$.
Then
\begin{itemize}
\item $h(x,a) \geq 0$ for all $x$, and
\item $|h(x,a)-x^2|<\nu$ for $|x|\leq 2R$, and
\item $h(x,a) \geq 4R^2-\nu$ for $|x|> 2R$.
\end{itemize}
\end{lemma}

\begin{proof}
 From Assumption~\ref{assum:2nd_derivative}, for $|x-\tau|\leq r$,  $\sigma$ has three derivatives and in this neighborhood, there is an $M'=2M/\sigma^{(2)}(\tau)$ such that $|\sigma^{(3)}(x)| \leq M = M'\sigma^{(2)}(\tau)/2$. Then in the interval $|x|<ar$,  $h(x,a)$ has three derivatives and the third derivative is bounded by $M'/a$. Moreover,  $h(0,a) = h'(0,a) = 0$, and the 2nd-order Taylor polynomial of $h(x,a)$ at zero is simply $T_2(x) = x^2$. From Taylor's theorem, it follows that
$$
|h(x,a)-x^2| \leq \frac{M'}{3! a }|x|^3
$$
on $|x|\leq ar$.

Now, $|x| \leq 2R$ implies $|x| \leq ar$. Using Taylor's theorem and the second constraint on $a$, we have
$$
|h(x,a)-x^2)| \leq \frac{M'\cdot 3\nu}{6 \cdot 4M'R^3}|x|^3 \leq \nu
$$
as desired.
Also, on this interval, using Taylor's theorem and the third constraint on $a$, we have
$$
h(x,a) \geq |x|^2- \frac{M'\cdot 3}{6 \cdot 2M'R}|x|^3 = |x|^2(1-|x|/(4R)) \geq (1/2)|x|^2 \geq 0
$$
showing $h$ is non-negative on this interval.

To show $h(x,a) \geq 4R^2-\nu$ for $|x| > 2R$, we rely on Assumption~\ref{assum:monotonic_sum}, which implies that $h(x,a)$ is monotonically increasing for $x\geq 0$. For $|x| > 2R$, we have $h(x,a) = h(|x|,a) \geq h(2R,a) \geq 4R^2-\nu$.
\end{proof}

Summing the outputs of $n$ supernodes, we obtain $\hat{g}(\xv) = \sum_{i=1}^n h(x_i,a)$, which is an approximation to $g(\xv)$,  as expressed in the following corollary.

\begin{corollary}\label{cor:first_layer_g}
Given a function $h(x,a)$ and associated range $R>0$ and accuracy $\nu>0$ as defined in Lemma~\ref{lemma:first_layer}, the function
$\hat{g}(\xv) = \sum_{i=1}^n h(x_i,a)$
satisfies
\begin{itemize}
\item $\hat{g}(\xv)\geq 0$
\item $|\hat{g}(\xv)-g(\xv)| \leq n\nu$ when $g(\xv) \leq 4R^2$
\item $\hat{g}(\xv) \geq 4R^2-n\nu$ when $g(\xv) > 4R^2$
\end{itemize}
\end{corollary}
\begin{proof}The first statement is trivial.
To see the second statement, suppose $g(\xv)\leq 4R^2$. Then $|x_i| \leq 2R$ for each $i$ and $|h(x_i,a)-x_i^2| \leq \nu$ for each $i$, yielding the result. For the second statement, suppose $g(\xv) > 4R^2$. If $|x_i| > 2R$ for some $i$, then $h(x_i,a) \geq 4R^2-\nu$, and non-negativity of $h$ implies $\hat{g}(\xv) \geq 4R^2-\nu \geq 4R^2-n\nu$. On the other hand, if $|x_i|< 2R$ for all $i$, then $h(x_i,a) \geq x_i^2-\nu$ for all $i$, and so $\hat{g}(\xv) \geq g(\xv)-n\nu$.
\end{proof}

\begin{lemma}\label{lemma:second_layer}
Given an activation function $\sigma$, defined in \eqref{eq:def-Psi},  satisfying Assumption~\ref{assum:exp_decay} with $\eta=1$,   a range $T > 0$ and accuracy $\epsilon>0$, construct the supernode function $\psi(x,\sigma)$ using $K$ basic nodes, where $K$ is chosen so that $\Delta = T/K$ satisfies   $\Delta \leq \log(1+\epsilon/40)$ and $\Delta < 1/2$.
Then this function satisfies
\begin{itemize}
\item $|\psi(x,\sigma)-\exp(-x)| \leq \epsilon \exp(-x)$ for $0\leq x \leq T$
\item $0 \leq \psi(x,\sigma) \leq \exp(-T)(1+\epsilon)$ for $x \geq T$
\end{itemize}
\end{lemma}
\begin{proof}
 Note that if $\sigma(x)$ is an indicator function $u(x) := 1\{x\geq 0\}$, this construction simply gives a piecewise-constant step approximation to $\exp(-x)$ over the interval $[0,T]$. The relative accuracy of such an approximation is uniform over the interval when using steps of fixed width $\Delta$.In particular, in the interval $k\Delta \leq x \leq (k+1)\Delta$, with $k\leq K$, the function $\psi(x,u)$ is equal to $\exp(-k\Delta)$, and the worst-case relative error for $x\in[0,T]$ is $(\exp(-k\Delta)-\exp(-(k+1)\Delta)/\exp(-(k+1)\Delta) = \exp(\Delta)-1$. Then $\Delta \leq \log(1+\epsilon/40)$, ensures that the relative error between $\psi(x,u)$ and $\exp(-x)$ is no more than $\epsilon/40$ on $[0,T]$.

For a general activation function we can write
$$
|\psi(x,\sigma)-\exp(-x)| \leq |\psi(x,u)-\exp(-x)| + |\psi(x,\sigma)-\psi(x,u)|.
$$
The relative accuracy of $\psi(x,\sigma)$ is thus assured if we can show e.g. that $|\psi(x,\sigma)-\psi(x,u)| \leq \epsilon/2 \exp(-x)$ on $[0,T]$.
By virtue of Assumption~\ref{assum:exp_decay}, we know that $|\sigma(x)-u(x)| \leq \exp(-|x|)$. Hence we can write
$$
|\psi(x,\sigma)-\psi(x,u)|  \leq \left(e^\Delta-1\right) \sum_{k=1}^K e^{-|x/\Delta-k|-k\Delta}.
$$

The sum can be further upperbounded by summing over all $-\infty < k < \infty$. Denoting $m = \lfloor x/\Delta \rfloor$ and $\nu = x/\Delta-m$, we can break the sum into ranges $k\leq m$ and $k>m$. The first range gives
\begin{eqnarray*}
\sum_{k\leq m} e^{-m-\nu+(1-\Delta)k} & \leq & e^{-\nu-\Delta m} \sum_{j\geq 0}  e^{-(1-\Delta)j} \\
& \leq & \frac{e^{-x+\Delta}}{1-e^{-(1-\Delta)}} = e^{-x}\frac{e^{\Delta}}{1-e^{-(1-\Delta)}} .
\end{eqnarray*}
Likewise, the second range gives
\begin{eqnarray*}
\sum_{k > m} e^{m+\nu-(1+\Delta)k} & \leq & e^{-(1+\Delta)+\nu-\Delta m} \sum_{j\geq 0}  e^{-(1+\Delta)j} \\
& \leq & e^{-x} \frac{ e^{1} }{1-e^{-(1+\Delta)}} \leq e^{-x} \frac{ e^{1} }{1-e^{-1}}
\end{eqnarray*}
Using $\Delta < 1/2$, we can combine terms to get
$$
|\psi(x,\sigma)-\psi(x,u)|  \leq \left(e^\Delta-1\right) \frac{2e^{1}}{1-e^{-1/2}} e^{-x}
$$
Since $\Delta \leq  \log(1+(\epsilon/40)) \leq \log(1+(\epsilon/2)(1-e^{-1/2})/(2e)) $,  we obtain $|\psi(x,\sigma)-\psi(x,u)|  \leq \epsilon/2 e^{-x}$, which completes the proof of relative accuracy on $[0,T]$.

To show that $\psi(x,\sigma) \leq \exp(-T)(1+\epsilon)$ for $x\geq T$, we note that $\psi(x,u)  = exp(-T)$ for $x\geq T$. We have already shown that $|\psi(x,\sigma)-\psi(x,u)| \leq \epsilon/2 e^{-x}$ for all $x$, and hence $|\psi(x,\sigma)-\psi(x,u)| \leq \epsilon/2 e^{-T}$ for $x\geq T$ in particular.
\end{proof}

\begin{lemma}\label{lemma:composed_accuracy}
Given desired accuracy $0<\delta<2$ and level $0 < \lambda < 1/(1+{\delta\over 4})$,
\begin{itemize}
\item Let the function $\hat{g}(\xv) = \sum_{i=1}^n h(x_i,a)$ be the function defined in Corollary~\ref{cor:first_layer_g} with range parameter $R=\sqrt{\log(1/\lambda)}$ and accuracy parameter $\nu = 1/n \log(1+\delta/4)$.
\item Let $\psi(x,\sigma)$ be a function satisfying conditions of Lemma~\ref{lemma:second_layer}
with range parameter $T = 4R^2$ and accuracy $\epsilon = \delta/2$.
\item Define $\hat{c}(\xv) = \psi(\hat{g}(\xv),\sigma)$.
\end{itemize}
 Then
\begin{itemize}
\item $|\hat{c}(\xv)-c(\xv)| < \delta c(\xv)$ whenever $c(\xv)\geq \lambda$
\item $\hat{c}(\xv) < \lambda (1+\delta)$ whenever $c(\xv) < \lambda$.
\end{itemize}
\end{lemma}

\begin{proof}
First, let us suppose that $c(\xv) \geq \lambda$. This implies $g(\xv) \leq R^2 < 4R^2$. Hence, by Corollary~\ref{cor:first_layer_g}, $|\hat{g}(\xv)-g(\xv)| \leq n\nu = \log(1+\delta/4)$, and so $|\exp(-\hat{g}(\xv))-c(\xv)| \leq \delta/4\, c(\xv)$. Moreover, since $\hat{g}(\xv) \leq g(\xv)+\log(1+\delta/4)$,  $g(\xv)\leq R^2$, and $\log(1+\delta/4)\leq \log{1\over \lambda} \leq R^2$,   we have $\hat{g}(\xv)\leq  2 R^2 \leq  T$.
Thus by Lemma~\ref{lemma:second_layer}, $|\hat{c}(\xv)-\exp(-\hat{g}(\xv))| = |\psi(\hat{g}(\xv),\sigma)-\exp(-\hat{g}(\xv))| \leq \delta/2 \exp(-\hat{g}(\xv))$.
Thus
\begin{eqnarray*}
|\hat{c}(\xv)-c(\xv)| & \leq & |\hat{c}(\xv)-\exp(-\hat{g}(\xv))|+|\exp(-\hat{g}(\xv))-c(\xv)|\\
& \leq& {\delta\over 4} \,c(\xv) + {\delta\over 2} \exp(-\hat{g}(\xv)) \\
& \leq & {\delta\over 4}  c(\xv) + {\delta \over 2}(1+{\delta \over 4})c(\xv)  \\
& \leq & \delta\, c(\xv),
\end{eqnarray*}
where in the last step we use $\delta < 2$.

Secondly, suppose that $c(\xv) < \lambda$, so that $g(\xv) > R^2$.  Whether or not $g(\xv) \geq 4R^2$, Corollary~\ref{cor:first_layer_g} implies that $\hat{g}(\xv) \geq R^2 - n\nu = -\log(\lambda(1+\delta/4))$, so that $\exp(-\hat{g}(\xv)) \leq \lambda(1+\delta/4)$.
If $\hat{g}(\xv) \leq T$, then (second layer result, Lemma~\ref{lemma:second_layer}) gives $\psi(\hat{g}(\xv)) \leq \epsilon\, \exp(-\hat{g}(\xv)) \leq \lambda(1+\delta/4)\delta/2 \leq \lambda (1+\delta)$ assuming $\delta < 4$. Or, if $\hat{x}(\xv) > T$, Lemma~\ref{lemma:second_layer} gives that $\psi(\hat{g}(\xv)) \leq \exp(-T)(1+\delta/2) \leq \exp(-R^2)(1+\delta/2) \leq \lambda(1+\delta)$ as desired.
\end{proof}

\begin{lemma}\label{lemma:chi-squared}
Let $\xv$ be a multi-dimensional Gaussian random variable in $\Real^n$ with pdf $p(\xv)$,
and suppose that $\Var(X_i) \leq V$ for each component.
Then
$$
P\left[ p(\xv) < t\right] \leq t^4 \exp\left(\frac{n}{8}\log(32\pi V)\right)
$$
\end{lemma}
\begin{proof}
The pdf is expressed $p(\xv)= \left( (2\pi)^n |\Sigma| \right)^{-1/2} \exp(-g(\xv))$ where $g(\xv) = (\xv-\mu)^T \Sigma^{-1} (\xv-\mu)$. Hence
$$
\P[p(\xv)<t] = \P[g(\xv) > \log(1/t) - \frac{1}{2}\log(| \Sigma|) - \frac{n}{2}\log(2\pi)].
$$
The eigenvalues of $\Sigma$ are positive and bounded as $\sum_{i=1}^n \lambda_i = \mbox{trace}(\Sigma)   \leq nV$. Maximizing $|\Sigma| = \Pi_i \lambda_i$ under this constraint, via Lagrange multipliers, yields $| \Sigma| \leq V^n$. Thus
$$
\P[p(\xv)<t] \leq \P[g(\xv) > \log(1/t) - \frac{n}{2}\log(2\pi V)].
$$
For any $\xv\sim \N(\mvec{\mu},\Sigma)$, $g(\xv)$ is a standard chi-squared variate with $n$ degrees of freedom. The Chernoff bound for such a variable can be expressed as
$$
\P\left[ g(\xv) \geq (1+\theta)n\right] \leq \exp\left( -\frac{n}{2}(\theta-\log(1+\theta))\right).
$$
Using a tangent bound to a convex function at $\theta=1$, we have $\theta-\log(1+\theta) \geq (\theta+1)/2 - \log(2)$, so that
$$
\P\left[ g(\xv) \geq (1+\theta)n\right] \leq \exp\left( -\frac{n}{4}(\theta+1-2\log(2))\right).
$$
and
$$
\P\left[ g(\xv) \geq s\right] \leq \exp\left( -\frac{1}{4}(s-2n\log(2))\right).
$$
Putting the bounds together, yields
$$
\P\left[p(\xv)<t\right] \leq \exp\left( -\frac{1}{4}\left( \log(1/t) - \frac{n}{2}\log(2\pi V)-\frac{n}{2}\log(16)\right)\right)
\leq t^4 \exp\left(   \frac{n}{8}\log(32\pi V)\right)
$$
as desired.
\end{proof}

We can now extend the analysis to a GMM.
\begin{lemma}\label{lemma:tstar}\label{lemma11}
Let $p(\xv) = \sum_{j=1}^J \alpha_j p_j(\xv)$ be the pdf of an $n$-dimensional GMM, i.e. $\sum_j \alpha_j = 1$ and each $p_j(\xv)$ a Gaussian distribution, and define $\xv$ to be a random variable with distribution $p$. 
Assume bounded variance $\int (x_i-\mu_{j,i})^2 p_j(\xv)\,d\xv \leq V$ for each element $i$ of each Gaussian distribution $p_j$.
Given $q>0$, choose $t^*>0$ such that
$$
\log(1/t^*) \geq \frac{n}{32}\log(32\pi V) + \frac{1}{4}\log(1/q) + \log(J).
$$
Then
$$
\P[p(\xv) < t^*] \leq q
$$
\end{lemma}
\begin{proof}
Choose $k = \arg \max_j \alpha_j$; we have $\alpha_k \geq 1/J$. 
Now applying Lemma~\ref{lemma:chi-squared} to $p_k$, we have
\begin{eqnarray*}
\P\left[p(\xv)<t^*\right] & \leq & \P\left[\alpha_k p_k(\xv)< t^*\right] \\
 & \leq & \P\left[p_k(\xv)< J t^*\right] \\
 & \leq & (Jt^*)^4 \exp\left( \frac{n}{8}\log(32\pi V) \right) 
\end{eqnarray*}
The assumed condition on $\log(1/t^*)$ then implies the result.
\end{proof}

\section{Proof of Theorem \ref{theor:main_two_layer}}\label{app:proof-thm-1}



We provide an explicit construction of a two-hidden layer subnetwork for approximating a given $c_j(\xv)$, and show that under Assumptions~\ref{assum:2nd_derivative}, \ref{assum:monotonic_sum}, and \ref{assum:exp_decay}, the  constructed network is accurate enough to satisfy the conditions of Lemma~\ref{lemma:decomposition}.
This is done by showing that super-nodes in the first layer approximate $x^2$ well enough, and a super-node in the second layer approximates $\exp(-x)$ well enough.


In the rest of the proof we drop the subscript $j$ and focus on a $c(\xv) = \beta \exp(-g(\xv))$, where $g(\xv) = \sum_i y_i^2$ with $\yv = \Sigma^{-1/2}(\xv-\mvec{\mu})$. Given parameters $\delta>0$ and $\lambda > 0$,  we  construct an approximation $\hat{c}$ of $c$, such that i) $|\hat{c}(\xv)-c(\xv)| < \delta c(\xv)$, whenever $c(\xv)\geq \lambda$, and ii)  $|\hat{c}(\xv)-c(\xv)| < \lambda (1+\delta)$, whenever $c(\xv) < \lambda$. We  show that the total  number of nodes used  by both hidden layers is $O(n)$.


To simplify the following steps, we normalize by $\beta$ to obtain $\tilde{c}(\xv) = c(\xv)/\beta$ which is to be approximated accurately above level $\tilde{\lambda} = \lambda/\beta$. If $\tilde{\lambda}\geq 1$, then it is sufficient to simply take the trivial approximation $\hat{\tilde{c}}(\xv) = 0$. Therefore in the following we assume $\tilde{\lambda} < 1$. 
With notation thus simplified, we seek to approximate $\tilde{c}(\xv) = \exp(-\sum_i y_i^2)$.

The condition $\tilde{c}(\xv) \geq \tilde{\lambda}$  corresponds to $\sum_i y_i^2 \leq \log(1/\tilde{\lambda})$. We will define $R = \sqrt{\log(1/\tilde{\lambda})}$, so that $\sum_i y_i^2 \leq R^2$  defines the region over which we must approximate $\tilde{c}(\xv)$ with small relative error.\\\\
\textit{Uniform approximation of $x^2$}. In the first hidden layer, we replace each node with activation function $x^2$ in the ideal reference model, with a supernode formed from two basic nodes with activation function $\sigma$ satisfying Assumptions~\ref{assum:2nd_derivative} and \ref{assum:monotonic_sum}. We use a special case of the construction in \cite{lin2017does}. In particular, we define the supernode as
\begin{align}
h(x,a) = \frac{a^{2}}{\sigma^{(2)}(\tau)}\left(\sigma(x/a+\tau) + \sigma(-x/a+\tau) - 2\sigma(\tau)\right).\label{eq:def-h}
\end{align}
The idea behind the construction is that the first term in the Taylor series of $h(x,a)$ at $x = 0$ is $x^2$, and so for sufficiently large $a$, the function approximates $x^2$ closely over the required range. Particularly, by Lemma \ref{lemma:first_layer}, choosing  $a\geq \max(2R/r,8MR^3/(3\nu \sigma^{(2)}(\tau)),4MR/(3\sigma^{(2)}(\tau)))$, we have 
\begin{itemize}
\item $h(x,a) \geq 0$ for all $x$, and
\item $|h(x,a)-x^2|<\nu$ for $|x|\leq 2R$, and
\item $h(x,a) \geq 4R^2-\nu$ for $|x|> 2R$.
\end{itemize}
Summing the outputs of $n$ supernodes, we obtain 
\begin{align}
\hat{g}(\xv) = \sum_{i=1}^n h(y_i,a).\label{eq:g-hat}
\end{align}
Then, by Corollary \ref{cor:first_layer_g},  $\hat{g}(\xv)$ is an approximation to $g(\xv)$, such that
\begin{itemize}
\item $|\hat{g}(\xv)-g(\xv)| \leq n\nu$, when $g(\xv) \leq 4R^2$,
\item $\hat{g}(\xv) \geq 4R^2-n\nu$, when $g(\xv) > 4R^2$.
\end{itemize}
\textit{Approximation of $\exp(-x)$}. After approximating the quadratic term $g(\xv)$ in the first layer, the role of the next level is to approximate $\exp(-x)$,  with a required level of accuracy $\epsilon$, over a required range $[0,T]$, using only $O(n)$ nodes.

For this construction, we rely on Assumption~\ref{assum:exp_decay}. The value of the bounding exponent $\eta$ in this assumption is not critical, since if $\sigma(x)$ satisfies the assumption with exponent $\eta$, the scaled function $\sigma(\alpha x)$ satisfies it with exponent $\alpha\eta$. To simplify notation, we take $\eta = 1$.

To form the super-node, we choose a sufficiently large number of components $K$, divide the range into intervals of width $\Delta = T/K$, and then form following sum of shifted activations:
\begin{align}
\psi(x,\sigma) = 1 + \sum_{k=1}^K \left(e^{-k\Delta}-e^{-(k-1)\Delta}\right)\sigma\left(\frac{x}{\Delta}-k\right).\label{eq:def-Psi}
\end{align}
This essentially constructs a staircase-like approximation to $\exp(-x)$.  Lemma  \ref{lemma:second_layer} states that if  $\Delta \leq\min( \log(1+\epsilon/40), 1/2)$, then
\begin{itemize}
\item $|\psi(x,\sigma)-\exp(-x)| \leq \epsilon \exp(-x)$, for $0\leq x \leq T$,
\item $0 \leq \psi(x,\sigma) \leq \exp(-T)(1+\epsilon)$, for $x \geq T$.
\end{itemize}

\textit{Accuracy of composed layers}. Consider $\hat{g}$ defined in \eqref{eq:g-hat},  with range parameter $R=\sqrt{\log(1/\tilde{\lambda})}$ and accuracy parameter $\nu = {1\over n} \log(1+\delta/4)$. Also, consider $\psi(x,\sigma)$ defined in \eqref{eq:def-Psi}, with range parameter $T = 4R^2$ and accuracy $\epsilon = \delta/2$. Define 
\begin{align}
\hat{\tilde{c}}(\xv) = \psi(\hat{g}(\xv),\sigma).
\end{align} By Lemma \ref{lemma:composed_accuracy}, 
\begin{itemize}
\item $|\hat{\tilde{c}}(\xv)-\tilde{c}(\xv)| < \delta \tilde{c}(\xv)$, whenever $\tilde{c}(\xv)\geq \tilde{\lambda}$,
\item $\hat{\tilde{c}}(\xv) < \tilde{\lambda} (1+\delta)$, whenever $\tilde{c}(\xv) < \tilde{\lambda}$.
\end{itemize}

We have thus established that, for arbitrarily small $\tilde{\lambda} > 0$, we can construct an approximation $\hat{\tilde{c}}(\xv)$ with the accuracy required for Lemma \ref{lemma:decomposition}. The same statement evidently holds for $\hat{c}(\xv) = \beta \hat{\tilde{c}}(\xv)$ with respect to $\lambda = \beta \tilde{\lambda}$.
Hence by Lemma \ref{lemma:decomposition}, we can construct a $(\delta,q)$ approximation $\hat{d}$ of $d$, for any $\delta$ and $q$.    It remains to show that that the number of nodes required, $M_n$, is $O(n)$.

\textit{Bounding the number of hidden nodes}. The final step is to show that the described network giving $\hat{c}(\xv)$ consists of $O(n)$ nodes. 

Exactly $2n$ nodes are needed in the first layer, since each of the $n$ functions $h_i(x,a)$ is constructed with two nodes.  In the second layer, the number of nodes is  $K = T/\Delta$. 
Since our network is designed with $\epsilon = \delta/2$, and since $\Delta \leq \min( \log(1+\epsilon/40), 1/2)$ is sufficient for Lemma \ref{lemma:second_layer}, the interval width $\Delta$ does not depend on $n$.
So, it remains to show that the range $T$ is $O(n)$.

At this step in the proof, we introduce notation related to the assumption that the eigenvalues of the covariances $\Sigma_j$ have constant upper and lower bounds. If the eigenvalues of $\Sigma_j$ are $\{\omega_i^j\}, i = 1,\ldots,n$, we require $\check{\omega} \leq \omega_i^j \leq \hat{\omega}$, where $0 < \check{\omega} \leq \hat{\omega}$ are the fixed bounds.  Intuitively, upper bounds relate to the typical case in practice that input distributions have bounded variance. The lower bound on eigenvalues prevents the model from approaching degenerate Gaussian distributions, or equivalently having arbitrarily sharp spatial detail.

Our network is designed with $T = 4R^2 = 4\log(1/\tilde{\lambda}) = 4\log(\beta) + 4\log(1/\lambda)$. Recall from Section~\ref{subsec:overview} that $\beta = \rho w (2\pi)^{-n/2} | \Sigma|^{-1/2}$ for fixed probabilities $\rho$ and $w$. From the assumed lower bound on eigenvalues of $\Sigma$, we have $|\Sigma|^{-1/2} \leq \left(1/\check{\omega}\right)^{n/2}$, and so $4\log(\beta)$ is $O(n)$. It remains to consider the scaling of $\log(1/\lambda)$ with $n$.

 For a given probability $q>0$, Lemma \ref{lemma:decomposition} requires that  $\lambda = t^* \delta/(2J(1+\delta)$,  where $t^*$ is such that $\P[d(X)<t^*] \leq q$, or equivalently, $\P[p(X)<  t^*/\rho ] \leq q$ where $\rho$ is the fixed prior probability. Since $J$ and $\delta$ are fixed, we need to show that $\log(1/t^*)$ is $O(n)$. The existence of such a $t^*$ is established by Lemma~\ref{lemma:tstar}. By assumption in Theorem~\ref{theor:main_two_layer}, the variance of each Gaussian distribution is upperbounded by $\hat{\omega}$. Hence 
 Lemma~\ref{lemma:tstar} shows that $t^*$ has the required scaling; specifically that as long as
 $$
\log(\rho/t^*) \geq \frac{n}{32}\log(32\pi \hat{\omega}) + \frac{1}{4}\log(1/q) + \log(J),
$$
then
$\P[p(X) < t^*/\rho] \leq q$. This concludes the proof of Theorem \ref{theor:main_two_layer}.


%

\section{Proof of Theorem \ref{thm:random-weights}}\label{appendix-last}
Consider random weights  $\wv_i\in\mathds{R}^n$, $i=1,\ldots,n_1$, that are mutually independent and distributed as $\Nc(\mvec{0}, I_n)$. 
By construction, given weights $\wv_1,\ldots,\wv_{n_1}$, we have
\begin{align}
f_{\wv}(\xv)&={\alpha^{n\over 4}\over n_1}\sum_{i=1}^{n_1} \cos\left({1\over \sqrt{{s}_f}}\langle \wv_i,\xv\rangle\right).
\end{align}
Here the subscript $\wv$ highlights the dependency of this function of the specific values of the weights.
For a fixed $\xv$, $\langle \wv_i,\xv\rangle$ is a zero-mean Gaussian random variable with variance $\|\xv\|^2$. Therefore, since $\wv_i$'s are i.i.d. themselves,  for a fixed $\xv$,  $\cos({1\over {s}_f}\langle \wv_i,\xv\rangle)$, $i=1,\ldots,n_1$, are i.i.d. bounded random variables.  Moreover,
\begin{align}
\E_{\wv}\left[ \cos\left({1\over \sqrt{{s}_f}}\langle \wv_i,\xv\rangle\right) \right]&={1\over \sqrt{2\pi}}\int \cos\left({\|\xv\|\over \sqrt{{s}_f}} u\right)\ex^{-{u^2\over 2}}du\nonumber\\
&\stackrel{\rm (a)}{=}\ex^{-{1\over 2{s}_f}\|\xv\|^2},
\end{align}
where $\rm (a)$ holds because of the following identity 
\[
\int \ex^{-a t^2} \cos(b t) dt = \sqrt{\pi \over a} \ex^{-{1\over 4a}b^2}.
\]
Therefore, for a fixed $\xv$,
\begin{align}
\E_{\wv}[f_{\wv}(\xv)]=\alpha^{n\over 4}\ex^{-{1\over 2{s}_f}\|\xv\|^2}=\mu_c(\xv).
\end{align}
The expected approximation error corresponding to  a fixed $\wv$ can be written as
\begin{align}
\E_{\xv}\left[(f_{\wv}(\xv)-\E_{\wv}[f_{\wv}(\xv)])^2\right].
\end{align}
But, for a fixed $\xv$,
\begin{align}
\E_{\wv}\left[(f_{\wv}(\xv)-\E_{\wv}[f_{\wv}(\xv)])^2\right]&={\alpha^{n\over 2}\over n_1} \var_{\wv}\left(\cos\left({1\over \sqrt{{s}_f}}\langle \wv_i,\xv\rangle\right) \right),\label{eq:Ew-var}
\end{align}
where
\begin{align}
\var_{\wv}\left(\cos\left({1\over \sqrt{{s}_f}}\langle \wv_i,\xv\rangle\right) \right)&={1\over \sqrt{2\pi}}\int \cos^2\left({\|\xv\|\over \sqrt{{s}_f}} u\right)\ex^{-{u^2\over 2}}du-\ex^{-{1\over {s}_f}\|\xv\|^2}\nonumber\\
&={1\over 2\sqrt{2\pi}}\int\Big(1+ \cos\left({2\|\xv\|\over \sqrt{{s}_f}} u\right)\Big)\ex^{-{u^2\over 2}}du-\ex^{-{1\over {s}_f}\|\xv\|^2}\nonumber\\
&={1\over 2} +{1\over 2}\ex^{-{2\over {s}_f}\|\xv\|^2}- \ex^{-{1\over {s}_f}\|\xv\|^2}.
\end{align}
Therefore, from \eqref{eq:Ew-var}, we have
\begin{align}
\E_{\wv}\left[(f_{\wv}(\xv)-\E_{\wv}[f_{\wv}(\xv)])^2\right]&={\alpha^{n\over 2}\over n_1} \left({1\over 2} +{1\over 2}\ex^{-{2\over {s}_f}\|\xv\|^2}- \ex^{-{1\over {s}_f}\|\xv\|^2}\right).\label{eq:Ew-var-2}
\end{align}
Taking the expected value of both sides with respect of $\xv$, and applying the Fubini's theorem (see e.g.~\cite{dibenedetto2002real}) to the left hand side, it follows that
\begin{align}
\E_{\wv}\left[\E_{\xv}[\left[(f_{\wv}(\xv)-\E_{\wv}[f_{\wv}(\xv)])^2\right]\right]&={\alpha^{n\over 2}\over 2n_1} +{\alpha^{n\over 2}\over 2n_1}\E_{\xv}[\ex^{-{2\over {s}_f}\|\xv\|^2}]- {\alpha^{n\over 2}\over n_1}\E_{\xv}[\ex^{-{1\over {s}_f}\|\xv\|^2}]\nonumber\\
&={\alpha^{n\over 2}\over 2n_1}\left(1+\Big({{s}_f\over {s}_f+4{s}_x}\Big)^{n\over 2}-2\Big({{s}_f\over {s}_f+2{s}_x}\Big)^{n\over 2}\right)\nonumber\\
&<{\alpha^{n\over 2}\over n_1}.
\end{align}
But by assumption $n_1>{1\over \epsilon}\alpha^{n\over 2}$, therefore,
\begin{align}
\E_{\wv}\left[\E_{\xv}[\left[(f_{\wv}(\xv)-\E_{\wv}[f_{\wv}(\xv)])^2\right]\right]< {\epsilon},
\end{align}
This shows that there exists at least one set of weights $\wv_1,\ldots,\wv_{n_1}$ which satisfies the desired error bound.\\\\

\bibliography{refs_list}
\end{document}